%% file: AdvML_Frontiers_2024.tex
\documentclass{article}

 \usepackage[final]{AdvML_Frontiers_2024}

\usepackage{natbib}
\usepackage[textsize=tiny]{todonotes}
\usepackage{algorithm}
\usepackage{algorithmic}
\usepackage{tcolorbox}
\usepackage{amssymb}
\usepackage{amsthm}
\usepackage{url}            
\usepackage{booktabs}       
\usepackage{amsfonts}       
\usepackage{nicefrac}       
\usepackage{microtype}      
\usepackage{adjustbox}
\usepackage{colortbl}
\usepackage{animate}
\usepackage{cancel}

\usepackage{bm}
\usepackage{bbm}
\usepackage{color}
\usepackage{booktabs}
\usepackage{multirow}
\usepackage{graphicx}
\usepackage{subcaption}
\usepackage{caption}
\usepackage{mathtools}
\usepackage{amsthm}
\usepackage{wrapfig}
\usepackage{bbding}
\usepackage{graphicx,subcaption,ragged2e}
\usepackage{color}         
\usepackage{colortbl}
\usepackage[T1]{fontenc}    
\usepackage[%
colorlinks = true,
citecolor  = blue,
linkcolor  = blue,
urlcolor   = red,
unicode,
pagebackref
]{hyperref}

\usepackage{cleveref}
\usepackage{url}            
\usepackage{booktabs}       
\usepackage{amsfonts}       
\usepackage{nicefrac}       
\usepackage{microtype}      
\theoremstyle{plain}
\newtheorem{theorem}{Theorem}[section]
\newtheorem{proposition}[theorem]{Proposition}

\theoremstyle{definition}
\newtheorem{definition}[theorem]{Definition}

\theoremstyle{remark}

\newcounter{casestudy} 

\newcommand{\insight}[2]{%
	\vspace{-0.18cm}%
	\begin{tcolorbox}[colback=gray!5!white,leftrule=2.5mm, boxrule=0.5mm, width=\textwidth+3mm, enlarge left by=-1mm, enlarge right by=-1mm, size=title]
		\textbf{#1}: #2
	\end{tcolorbox}
	\vspace{-0.16cm}%
}
\definecolor{darkergreen}{RGB}{0, 0, 0}
\definecolor{red2}{RGB}{252, 54, 65}

\definecolor{Gray}{gray}{0.85}
\definecolor{bleudefrance}{rgb}{0.19, 0.55, 0.91}
\definecolor{LightGray}{gray}{1}
\definecolor{darkergreen}{RGB}{0, 0, 0}

\input{latex_defs.tex}

\newcommand{\vth}{\bm{\theta}}

\newcommand{\myparagraph}{\textbf}

\title{In Search of the \textit{Successful} Interpolation:\\ On the Role of \textit{Sharpness} in CLIP Generalization}

\author{%
  Alireza Abdollahpoorrostam \\
  Department of Computer Science $\&$ Communication Systems\\
  EPFL\\
  Switzerland, Lausanne\\
  \texttt{alireza.abdollahpoorrostam@epfl.ch} \\
}

\begin{document}

\maketitle

\begin{abstract}
\textit{Zero-shot} models like CLIP are often fine-tuned on a target dataset to improve its accuracy further, but this can compromise out-of-distribution (OOD) robustness. Robust Fine-Tuning (\texttt{RFT}
)~\citep{wortsman2021robust}, which interpolates between the \textit{zero-shot} and \textit{fine-tuned} models, has been proposed to address this issue. However, understanding when \texttt{RFT} actually improves OOD error remains limited. In this work, we empirically investigate the robustness of \texttt{RFT} in CLIP models, with a focus on the \textit{sharpness} of the CLIP model during interpolation. First, we demonstrate that while sharpness may not serve as a reliable indicator for predicting the generalization of modern architectures like CLIP on OOD data, this challenges the conventional belief in the generalization benefits of flat minima in foundation models. However, by examining the role of the \textit{straggler layer} phenomenon, we show that, unlike overall sharpness, the \textit{layer-wise} sharpness of \textit{straggler} layers can reliably capture the generalization performance of interpolated CLIP models on OOD data.
Our extensive experiments reveal that \textit{layer-wise} sharpness correlates with generalization in OOD accuracy for \texttt{RFT}. Furthermore, we demonstrate that by inducing sparsity in the \textit{straggler} layers, we can mitigate the \textit{failure mode} phenomenon in \texttt{RFT}. To the best of our knowledge, this is the first work to study the role of sharpness in the \textit{success} of interpolation in the weight space of CLIP foundation models. Our code is
available at \url{https://github.com/alirezaabdollahpour/CLIP_Mode_Connectivity}. \looseness=-1
\end{abstract}

\section{Introduction}
Understanding the behavior of large machine learning models like CLIP~\citep{radford2021learning} on OOD tasks is important for their safe deployment. Analyzing their behavior on a path between the initial and the final parameters has been proposed as a simple yet insightful approach this. However, prior works~\citep{vlaar2022can,lucas2021analyzing,neyshabur2020being,draxler2018essentially,entezari2021role,Chatterji2020} has primarily focused on CNN models for this analysis and whether such analysis extends to other kinds of architecture has not been thoroughly explored.
On the other hand, several works have shown that while foundation models like CLIP exhibit outstanding zero-shot OOD performance, this can be further improved if they are fine-tuned on the relevant target domain. However, this improvement comes at the cost of reduced performance on domains that it is not trained on. To solve this problem, inspired by the above-mentioned works on interpolation in CNNs,~\citet{RFT} showed that on the path connecting the \textit{zero-shot} model and the final \textit{fine-tuned} model, there exists a model with better OOD performance and proposed an algorithm called~\emph{Robust Fine Tuning}~(\texttt{RFT}) to find this parameter.
However, \texttt{RFT} does not always succeed in achieving large improvement in OOD accuracy compared to the \textit{zero-shot} model, and very little understanding exists of when the improvement is large and when it isn't. In this work, we aim to address this lack of knowledge. Inspired by earlier work on the interpolation between two CNN models, we first provide extensive experimental results to examine the correlation between the weight space geometry and CLIP's capability to generalize on OOD tasks. We aim to address the following question:
\vspace*{-7pt}
\begin{quote}
	\textit{How does sharpness on OOD samples relate to CLIP generalization?}
\end{quote}
\vspace*{-8pt}
Second, we investigate the role of the specific layer's sharpness on CLIP's OOD generalization. In particular, we ask the following question:
\vspace*{-7pt}
\begin{quote}
	\textit{What occurs within a layer during interpolation that leads to a \textit{failure mode}? By measuring the sharpness of that layer during interpolation, can we predict its impact on generalization?}
\end{quote}
\vspace*{-8pt}


\textbf{Robust Fine-Tuning}~(\texttt{RFT}) method has two steps: first, they fine-tune the \textit{zero-shot} model on the target distribution. Second,
they combine the original \textit{zero-shot} and fine-tuned models
by linearly interpolating between their weights, coined as
weight-space ensembling.
Nevertheless, the connection between linear interpolation and OOD generalization for CLIP has not been thoroughly investigated. The question of why the linear interpolation between~\textit{zero-shot} and fine-tuned CLIP models succeeds in OOD tasks, and the conditions under which the linear path between two CLIP models indicates robust generalization performance on OOD tasks, remains an unresolved problem. The recent advancements in the understanding of loss landscapes in CNNs and their connection to generalization through linear paths have prompted \cite{abdolahpourrostam2024unveiling} to revisit these findings within the context of foundation models like CLIP. \cite{abdolahpourrostam2024unveiling} aims to bridge the gap between the assumptions made about linear interpolation and loss landscape geometry in CNNs and the generalization capabilities of CLIP. Their study seeks to identify the conditions under which linear interpolation can be \textit{successfully} applied between two CLIP models, with particular attention to the roles of data augmentation and learning rate magnitude during the fine-tuning process.\looseness=-1

\myparagraph{On the role of sharpness:}
There is a body of literature suggesting that flatter minima may have better generalization properties \citep{xing2018walksgd, zhou2021theoreticallyunderstandingsgdgeneralizes, cha2021swaddomaingeneralizationseeking, park2022visiontransformerswork, lyu2023understandinggeneralizationbenefitnormalization, andriushchenko2023modernlookrelationshipsharpness} for standard or OOD data. However, the definitions of sharpness commonly used in the field do not align effectively with the concept of generalization, as discussed \citep{ kaur2023maximumhessianeigenvaluegeneralization} this can be primarily due to the model's lack of invariance under reparametrizations that not change the model \citep{dinh2017sharpminimageneralizedeep, granziol2020flatnessfalsefriend, zhang2021flatnessdoesdoescorrelate, andriushchenko2023modernlookrelationshipsharpness}. 
The utilization of adaptive sharpness seems to hold more potential as it effectively resolves the reparametrization problem and has been demonstrated to exhibit a stronger empirical correlation with generalization. \citep{kwon2021asamadaptivesharpnessawareminimization, andriushchenko2023modernlookrelationshipsharpness}. Furthermore, SAM demonstrates notable utility in emerging architectures such as vision transformers \citep{chen2022visiontransformersoutperformresnets, andriushchenko2023modernlookrelationshipsharpness}. In addition, although transfer learning has become the prevailing method for vision problems, the consequences of sharpness in this context have not been thoroughly investigated. Furthermore, the correlation between sharpness and OOD generalization has not been extensively examined. These rising innovations highlight the necessity to reevaluate the significance of sharpness in these new environments.

\subsection{Background on Interpolation and Notations}
\myparagraph{Loss barrier.}
For loss landscapes, \textit{barriers} refer to regions of increased loss encountered along the interpolation path between two sets of model parameters.\looseness=-1

We examine a CLIP architecture that is parametrized by $\mathcal{\vth}$ and is fine-tuned on a task represented by a training set $S_{\text{train}}$ and a test set $S_{\text{test}}$.
In the following, as we are interested in the generalization of CLIP on OOD tasks, we consider OOD loss and accuracy and write $\mathcal{L}(\mathcal{\vth}), \mathcal{A}(\mathcal{\vth})$ for $\mathcal{L}(\mathcal{\vth}, S_{\text{OOD}}), \mathcal{A}(\mathcal{\vth}, S_{\text{OOD}})$. Assume that we have fixed two different different sets of weights $\mathcal{\vth}_0$ and $\mathcal{\vth}_1$. Let $\mathcal{L}_{\alpha}(\mathcal{\vth}_0, \mathcal{\vth}_1) = \mathcal{L}(\alpha \mathcal{\vth}_0 + (1 - \alpha) \mathcal{\vth}_1)$ and $\mathcal{A}_{\alpha}(\mathcal{\vth}_0, \mathcal{\vth}_1) = \mathcal{A}(\alpha \mathcal{\vth}_0 + (1 - \alpha) \mathcal{\vth}_1)$ for $\alpha \in [0, 1]$ be the loss and accuracy, respectively, of the CLIP network created by linearly interpolating between $\mathcal{\vth}_0$ and $\mathcal{\vth}_1$. Then, building upon the~\cite{frankle2020linear} definition for linear interpolation instability, we define it for CLIP on OOD as the following notion.\looseness=-1

\myparagraph{Definition 1.}
\textit{The difference between the supremum of the loss for any interpolation $\sup_{\alpha} \mathcal{L}_{\alpha}(\mathbb{\vth}_0, \mathbb{\vth}_1)$ and the average loss of the endpoints $\frac{1}{2} (\mathcal{L}(\mathbb{\vth}_0) + \mathcal{L}(\mathbb{\vth}_1))$ is called the linear interpolation instability for the CLIP on OOD.}\looseness=-1

Recall that \textit{zero-shot} CLIP performs better on OOD tasks compared to the fine-tuned version of CLIP. Within the same settings of~\citep{RFT,wortsman2022model}, we are interested in exploring the linear path between \textit{zero-shot} CLIP and fine-tuned CLIP. Therefore, we set $\mathcal{\vth}_{0}$ as \textit{zero-shot} model.\looseness=-1

Two parametrizations $\mathcal{\vth}_0$ and $\mathcal{\vth}_1$ have a \textbf{barrier} between them if the linear interpolation instability for \textbf{\textit{sufficiently}} large $\delta$, there exists an $\alpha \in [0, 1]$ such that:
\begin{equation}
	\sup_{\alpha} \mathcal{L}_{\alpha}(\mathbb{\vth}_0, \mathbb{\vth}_1; S_{\text{OOD}}) - \mathcal{L}(\mathbb{\vth}_0; S_{\text{OOD}}) \geq \delta > 0
\end{equation}
The value of \(\delta\) can be empirically determined for each OOD task.
Similarly, we state that linear interpolation or the \texttt{RFT} algorithm can achieve \textbf{\textit{high gain accuracy}} if there exists an $\alpha \in [0, 1]$ such that:
\begin{equation}
	\label{eq:Acc}
	\sup_{\alpha} \mathcal{A}_{\alpha}(\mathcal{\vth}_0, \mathcal{\vth}_1; S_{\text{OOD}}) - \mathcal{A}(\mathcal{\vth}_0; S_{\text{OOD}}) \geq \xi > 0\vspace{-5pt}
\end{equation}

where \(\xi\) is \textbf{\textit{sufficiently}} large. 

Also, we define a linear path as having a \textit{gain} if the \textit{supremum} in Eq.~\ref{eq:Acc} exists with $(\xi > 0 )$. It is important to mention that a path is considered a \textit{\textbf{failure mode}} if the \textit{supremum} in Eq.~\ref{eq:Acc} does not exist. Figure~\ref{fig:Acc_Loss} illustrates scenarios in which several distinct \textit{fine-tuned} CLIP models experience either \textit{failure mode} or \textit{high gain accuracy} outcomes during the interpolation~(\texttt{RFT}).
\begin{figure}[h!] 
	\centering 
	\small
	\tabcolsep=1.0pt
	\begin{tabular}{@{}c@{}c@{}}
		\textbf{Failure Mode} & \textbf{High Gain Accuracy} \\
		\includegraphics[width=0.5\textwidth]{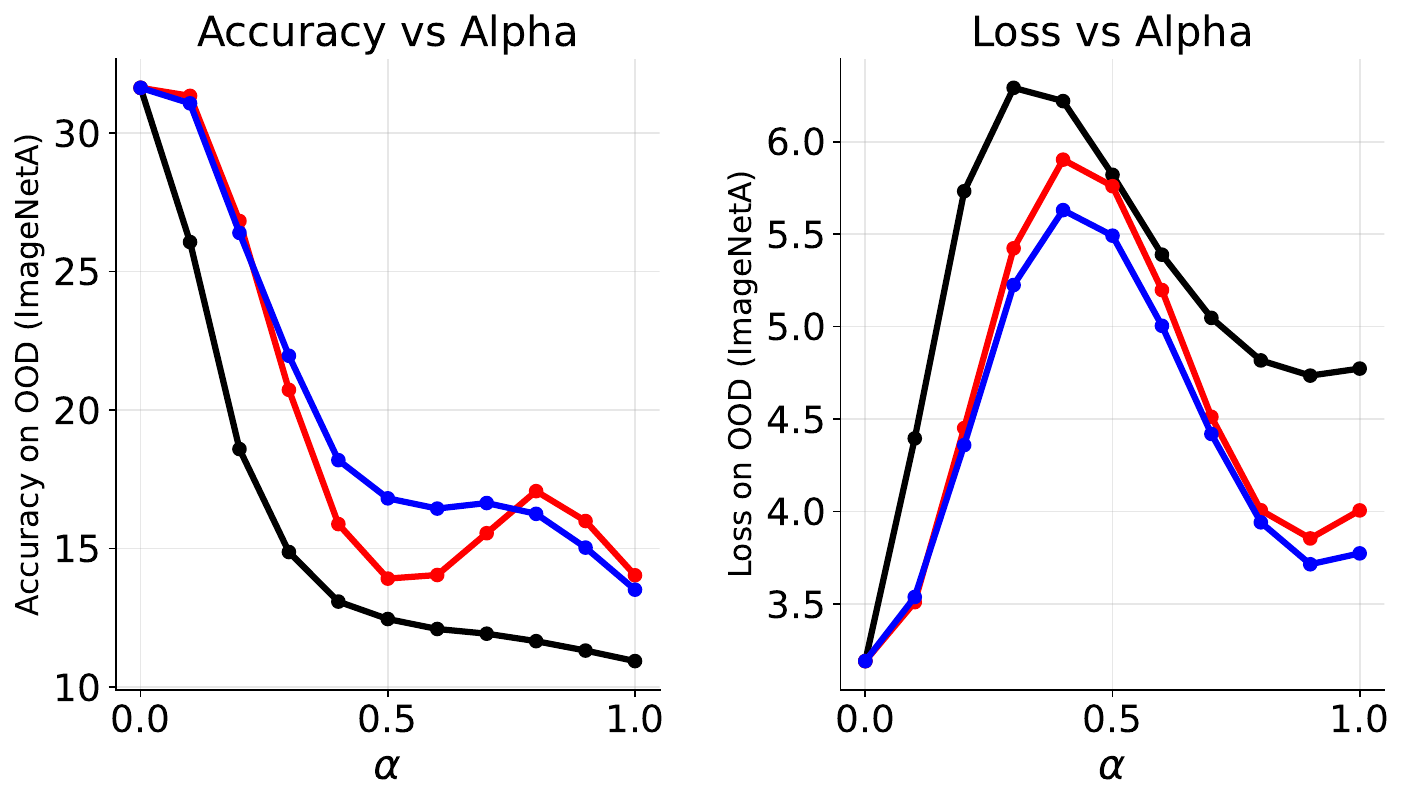} & 
		\includegraphics[width=0.5\textwidth]{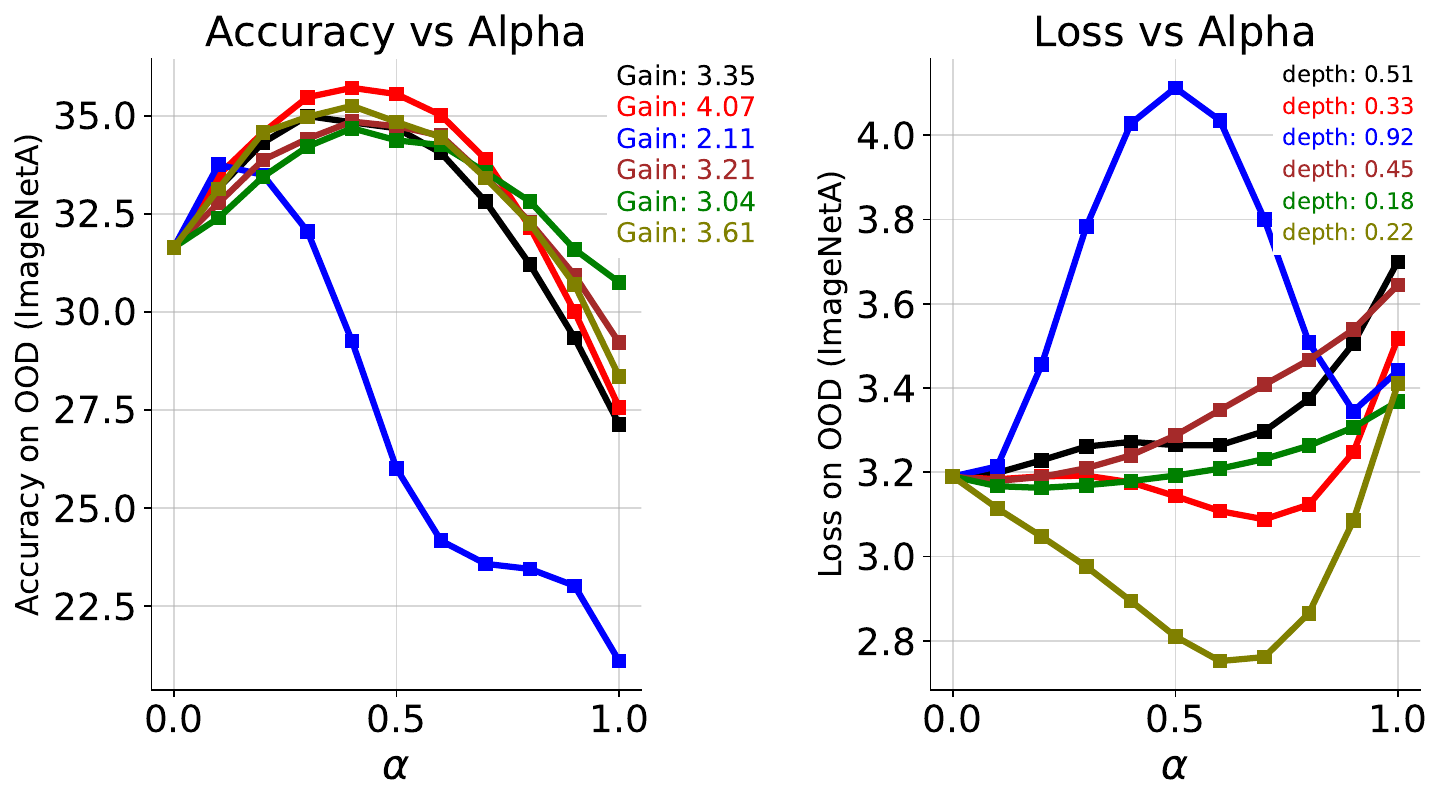} 
	\end{tabular}
	\vspace{-2.5mm}
	\caption{For 9 distinct fine-tuned CLIP models~(each color shows different CLIP models) on ImageNet~\citep{deng2009imagenet},
		this plot demonstrates the accuracy and loss on ImageNet-A~\citep{hendrycks2021natural} as an OOD task. For each model, we show the maximum accuracy gain achieved along a corresponding interpolation path. In the loss plot, we show \textbf{depth} as the largest barrier on the interpolation path starting from the \textit{zero-shot} model.}
	\label{fig:Acc_Loss}
\end{figure}

\myparagraph{Layer-wise interpolation.}
In the following, we analogously define a layer-wise notion of instability. 
Let $\mathcal{M}$ be structured in $L$ layers $\{\vW^{(1)}, \dots, \vW^{(L)}\}$.
In our experiments, we consider both weights and bias as one set of parameters describing a layer.
Let us fix a layer $\vW^{(i)}$.
Consider a parametrization that is defined by $\alpha$, $\mathcal{W}_1$ and $\mathcal{W}_2$ as $\{\vW^{(1)}_j, \vW^{(2)}_j, \dots, \alpha \vW^{(i)}_1 + (1-\alpha)\vW^{(i)}_2, \dots, \vW^{(L)}_j\}$ where $j$ can be selected to be $1$ or $2$.
\begin{definition}\textit{(Layer-wise linear interpolation instability)
	\label{def:one_lw}
	The difference between supremum of the loss on the line $\text{sup}_{\alpha} \mathcal{L}_{\alpha, i}(\mathcal{W}_1, \mathcal{W}_2)$ corresponding to layer $\vW^{(i)}$ and average loss of the original models $\frac{1}{2}(\mathcal{L}(\mathcal{W}_1)+\mathcal{L}(\mathcal{W}_2))$ is the \textbf{layer-wise linear interpolation instability} for the given architecture $\mathcal{M}$ and selected layer~(A similar approach can be employed to analyze this phenomenon by evaluating the accuracy on OOD data.).}
\end{definition}

\begin{definition}\textit{(Straggler layer)
If a layer demonstrates layer-wise interpolation instability, it is referred to as a straggler layer.}
\end{definition}

We are particularly interested in layers where linear interpolation leads to a \textit{\textbf{failure mode}} in terms of accuracy on OOD data. In other words, if a layer exhibits \textit{\textbf{layer-wise interpolation instability}}, it manifests this \textit{\textbf{failure mode}} phenomenon.

\textbf{Note}: Since we utilize the weights of the \textit{zero-shot} CLIP model, denoted as $\mathcal{W}_1$ ($\mathbb{\vth}_0$), and the weights $\mathcal{W}_2$ from the \textit{fine-tuned} CLIP model ($\mathbb{\vth}_1$ or $\mathbb{\vth}_{\text{FT}}$), we assign the \textit{zero-shot} CLIP weights to all layers except the target layer $i$. This approach allows us to specifically analyze the performance of layer $i$ in the \textit{fine-tuned} CLIP model.



\begin{figure}[t!] \centering\small
	\tabcolsep=1.1pt
	\newl=.495\columnwidth
	\begin{tabular}{c c} 
		\hspace{3mm} \textbf{Failure Mode} & \hspace{3mm} \textbf{High Gain Accuracy} \\
		\includegraphics[width=\newl]{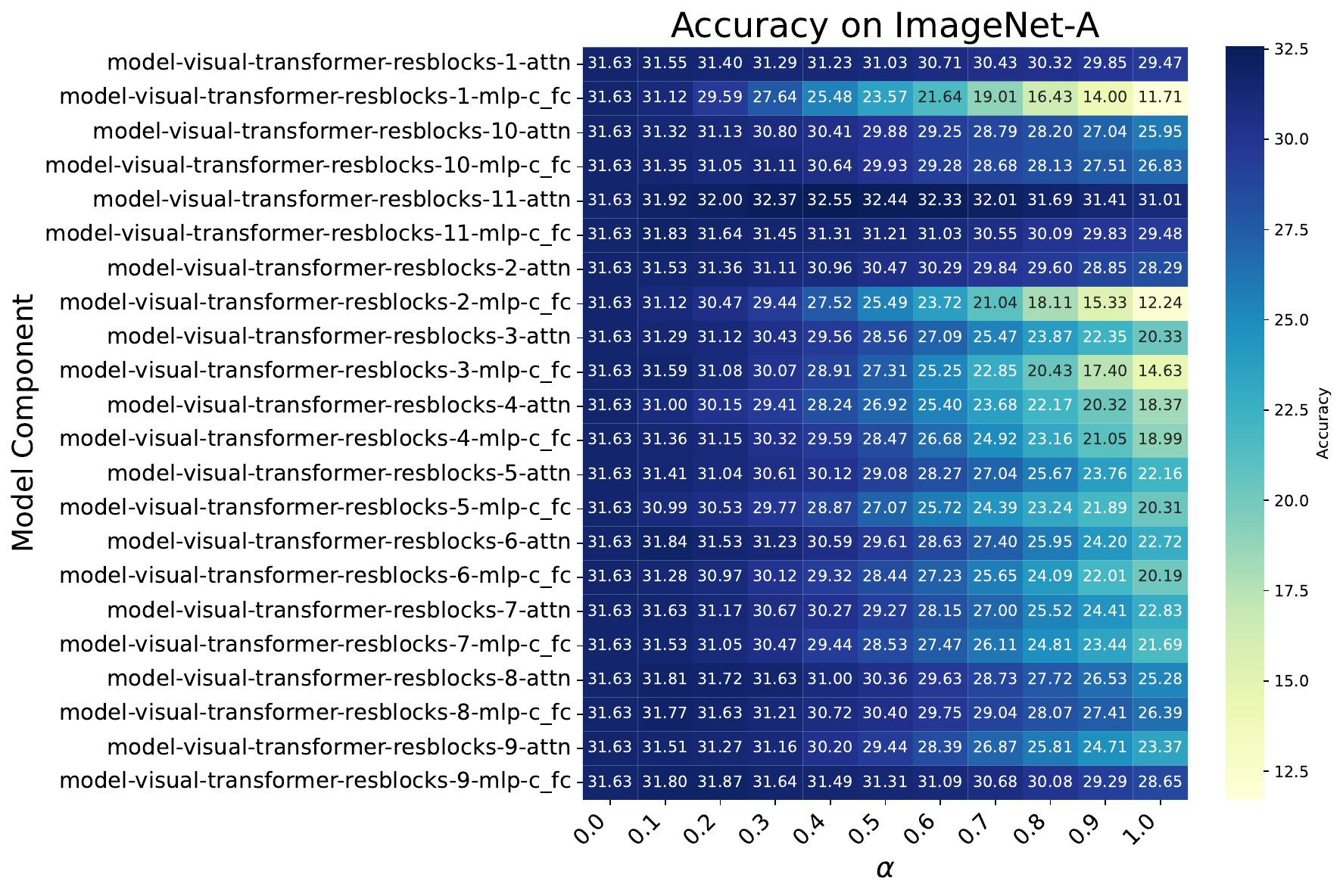} &
		\includegraphics[width=\newl]{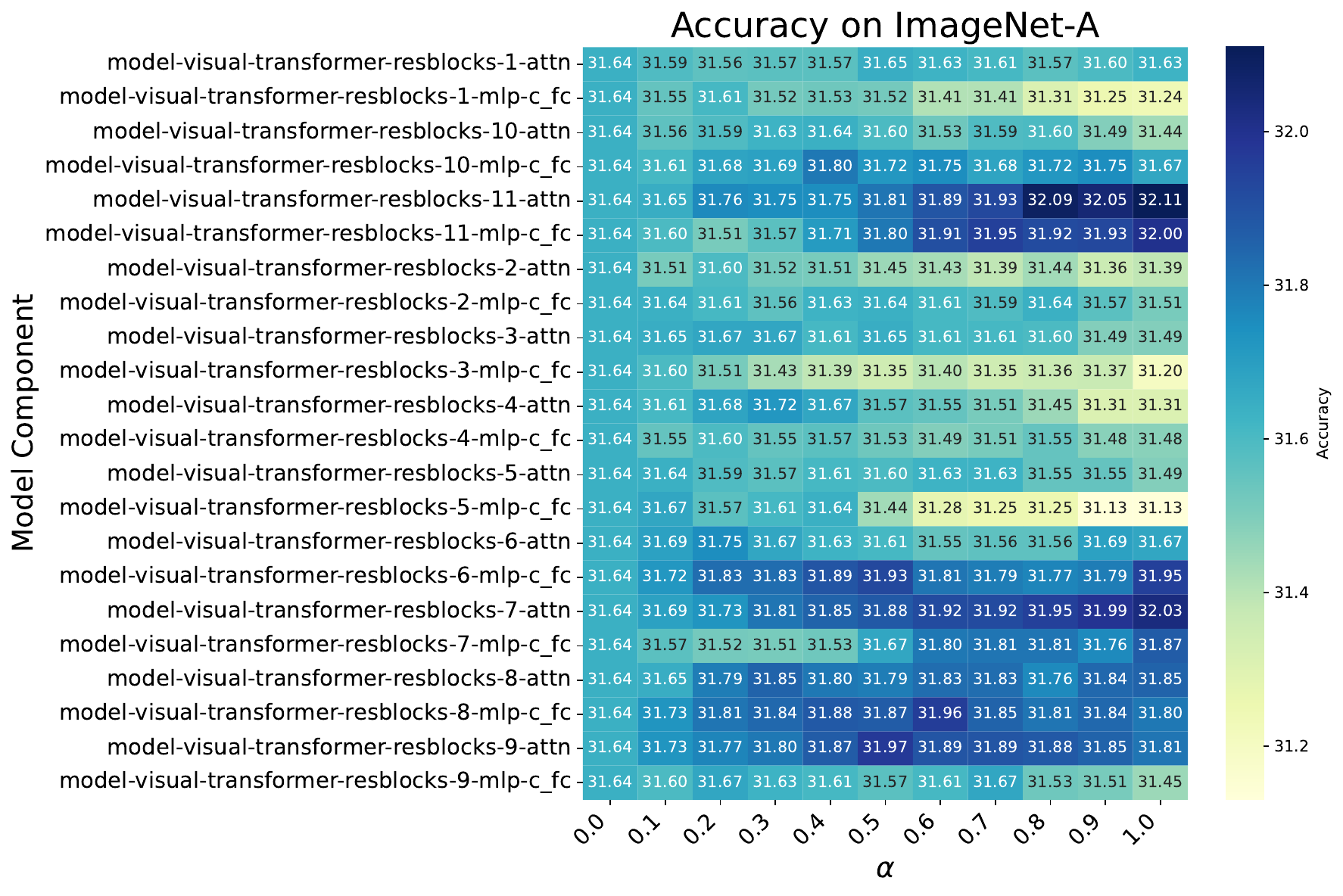}\\
	\end{tabular}
	\vspace{-4mm}
	\caption{\textbf{Layer-wise interpolation on ImageNet-A as OOD.} For two distinct fine-tuned CLIP models one exhibiting \textit{failure mode} and the other \textit{high gain accuracy} in regular interpolation~(\texttt{RFT}), we conduct a \textit{layer-wise} interpolation alongside each layer with the \textit{zero-shot} CLIP model.} 
	\label{fig:heatmap}
\end{figure}

\section{Adaptive Sharpness and its Invariances}
In this section, we begin by providing background on adaptive sharpness and then discuss its invariance properties in modern architectures. We categorize the sharpness of a model into two distinct categories. First, we establish a connection between \textit{general} sharpness and the generalization performance of the interpolated CLIP model. Second, we introduce the concept of \textit{layer-wise} sharpness and, by utilizing the relationship between \textit{straggler} layers and \textit{layer-wise} sharpness, we experimentally demonstrate how the \textit{layer-wise} sharpness of \textit{straggler} layers can capture the generalization performance of interpolated CLIP models.

\subsection{Background on Sharpness}
\myparagraph{Sharpness definitions.}
Similar to \citep{andriushchenko2023modernlookrelationshipsharpness}, we denote the loss on a set of \textit{OOD} points $\Scal$ as $L_\Scal(\wv) = \frac{1}{|S|} \sum_{(\xv, \yv) \in \Scal} \l_{\xv\yv}(\wv)$, where $\ell_{\xv \yv}(\wv) \in \R_+$ represents some loss function (e.g., cross-entropy) on the pair $(\xv, \yv) \in \Scal$ computed with the network weights $\wv$.
For arbitrary $\wv \in \R^p$ (i.e., not necessarily a minimum), we define the \textit{average-case} and \textit{adaptive average-case} sharpness with radius $\rho$ and with respect to a vector $\cv \in \R^p$ as:
\begin{align} \label{eq:sharpness}
	S_{avg}^\rho(\wv, \cv) &\triangleq \E_{\substack{\Scal \sim P_m \ \ \ \ \ \ \\ \deltav \sim \mathcal{N}(0, \rho^2 diag(\cv^2))}} \hspace{-8mm} L_\Scal(\wv + \deltav) - L_\Scal(\wv)
\end{align}
where $\odot$/$^{-1}$ denotes elementwise multiplication/inversion and $P_m$ is the data distribution that returns $m$ pairs $(\xv, \yv)$. 
Using $\cv = |\wv|$ leads to \textit{elementwise} adaptive sharpness \citep{kwon2021asamadaptivesharpnessawareminimization, andriushchenko2023modernlookrelationshipsharpness} and makes the sharpness invariant under multiplicative reparametrizations.
For a thrice differentiable loss $L(\wv)$, the average-case elementwise adaptive sharpness can be computed as (see \cite{andriushchenko2023modernlookrelationshipsharpness} or App.~\ref{sec:app_asymptotic} for proof):
\begin{align} \nonumber
	S_{avg}^\rho(\wv, |\wv|) = &\E_{\Scal \sim P_m} \frac{\rho^2}{2} \tr(\nabla^2 L_\Scal(\wv) \odot |\wv| |\wv|^\top) + O(\rho^3)
	\label{eq:avg_sharpness_small_rho}
\end{align}
We should also mention that the first-order term cancels out completely. 
In order for better clarity, we will use the term \textit{general sharpness}. In the upcoming sections, we will examine the connection between the sharpness of interpolated CLIP models and their generalization performance on OOD data. Next, we present our concept of \textit{layer-wise sharpness}, which entails quantifying the sharpness of \textit{\textbf{one specific layer}} within the CLIP model during interpolation.
\section{Sharpness vs. Generalization}

The current understanding of the relationship between sharpness and generalization is primarily based on experiments with non-residual convolutional networks and small datasets such as CIFAR-10 and SVHN~\citep{jiang2019fantasticgeneralizationmeasures}. \cite{andriushchenko2023modernlookrelationshipsharpness} were the first to study the correlation between \textit{general} sharpness and generalization in transformer-based modern architectures, such as fine-tuned CLIP models. Their findings revealed that there is no \textit{strong} correlation between \textit{general} sharpness and generalization on OOD data. Building on their observations, we investigate the correlation between \textit{general} sharpness and interpolation. Additionally, we introduce the concept of \textit{layer-wise} sharpness and demonstrate how, unlike \textit{general} sharpness, it can effectively capture the generalization performance during interpolation in weight space between \textit{zero-shot} and \textit{fine-tuned} CLIP models.

\begin{wrapfigure}[19]{r}{0.5\textwidth}
	\vspace{-0.65cm}
	\includegraphics[width=0.9\linewidth]{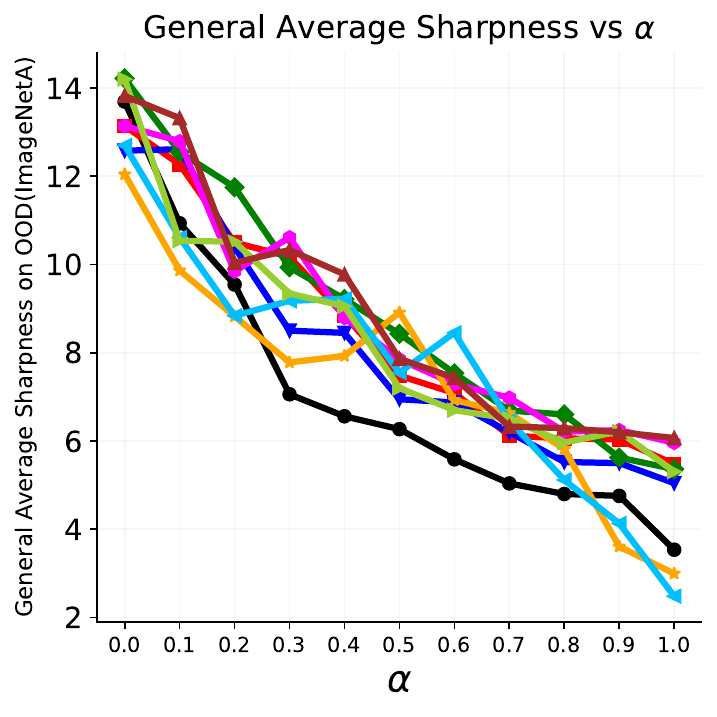}
	\vspace{-0.3cm} 
	\caption{For 9 distinct fine-tuned CLIP models~(each color shows different CLIP models) on ImageNet,
		this plot demonstrates the \textit{general adaptive average sharpness} with $\rho=1.0$ and $20$ iterations on ImageNet-A as an OOD task.}
	\label{fig:genral_sharpness}
\end{wrapfigure}
In Fig.~\ref{fig:genral_sharpness}, we demonstrate that \textit{general sharpness} \textit{\textbf{fails}} to directly capture the generalization of interpolated CLIP models on OOD data. Contrary to our expectations, CLIP models fine-tuned on ImageNet indicate that flatter solutions consistently generalize worse on OOD data. This evidence suggests that the commonly held belief in the generalization benefits of flat minima does not hold true in modern settings. This result corroborates the findings of \cite{andriushchenko2023modernlookrelationshipsharpness}, specifically for \textit{fine-tuned} CLIP models on OOD data.

\subsection{Layer-wise Sharpness}
In this part, we introduce the concept of \textit{layer-wise sharpness}, where we perturb the weight space of the target layer in the \textit{fine-tuned} CLIP model during interpolation. Subsequently, we perform interpolation between this newly perturbed \textit{fine-tuned} CLIP model and the \textit{zero-shot} CLIP model. Notably, we do not conduct \textit{layer-wise} interpolation; instead, we apply the previously described \texttt{RFT} algorithm. Informally speaking, we want answer to this question:
\insight{\textbf{\textit{Question}}}{\textit{What occurs within a layer during interpolation that leads to \textit{layer-wise interpolation instability} or a \textit{failure mode}? By measuring the sharpness of that layer during interpolation, can we predict its impact on generalization?}}
Furthermore, we empirically investigate what occurs immediately after $\alpha^{\star}$ in \textit{high gain accuracy} models. As shown in Fig.~\ref{fig:Acc_Loss}, for these models, we consistently observe a point along the interpolation path where the interpolated model reaches \textit{maximum} accuracy. Beyond this point, a decline in performance begins. 
\begin{figure*}[t!] 
	\centering 
	\small
	\begin{tabular}{@{}c@{}c@{}}
		\textbf{High Gain Accuracy} &
		\textbf{Failure Mode} \\
		\includegraphics[width=0.5\linewidth]{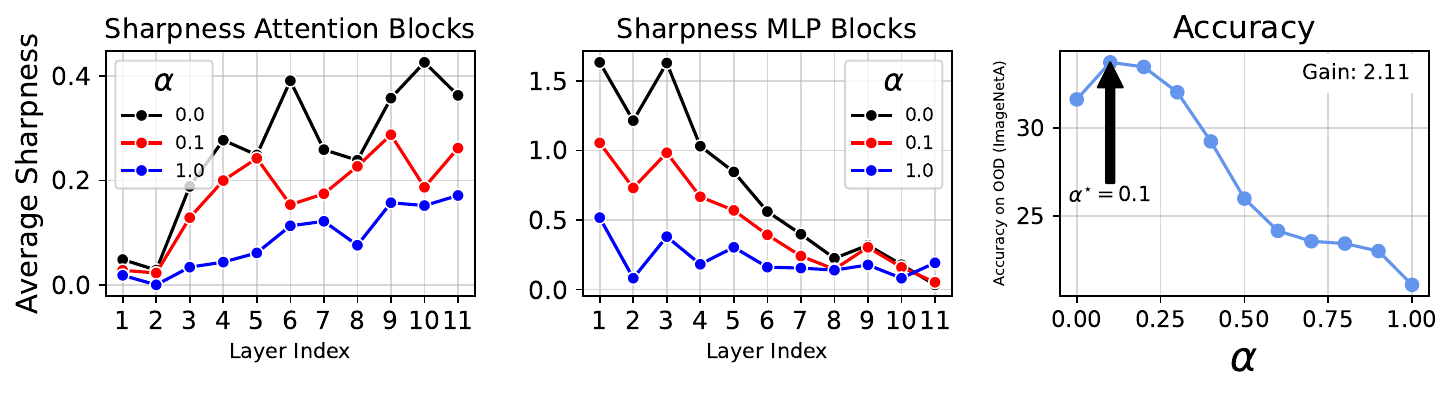} &
		\includegraphics[width=0.5\linewidth]{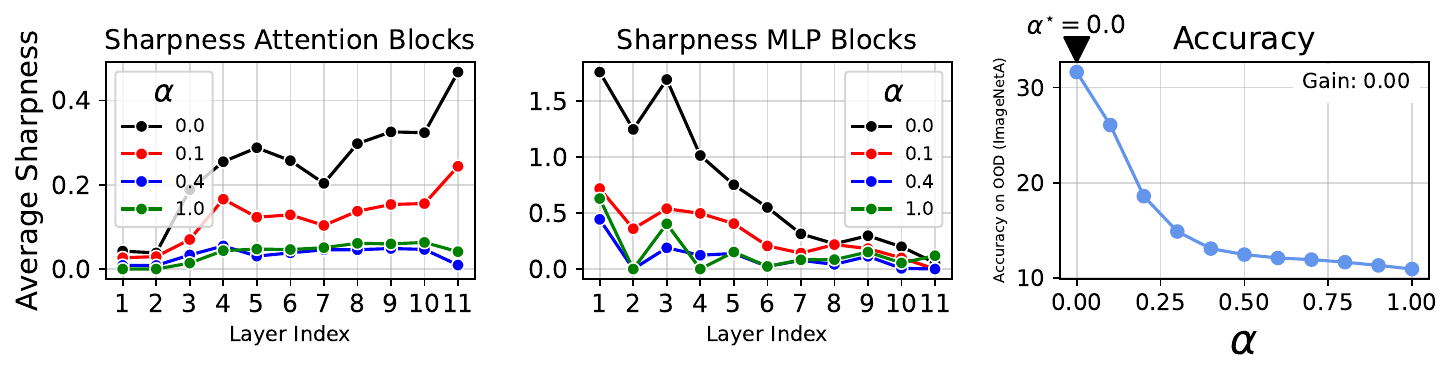} \\
		\includegraphics[width=0.5\linewidth]{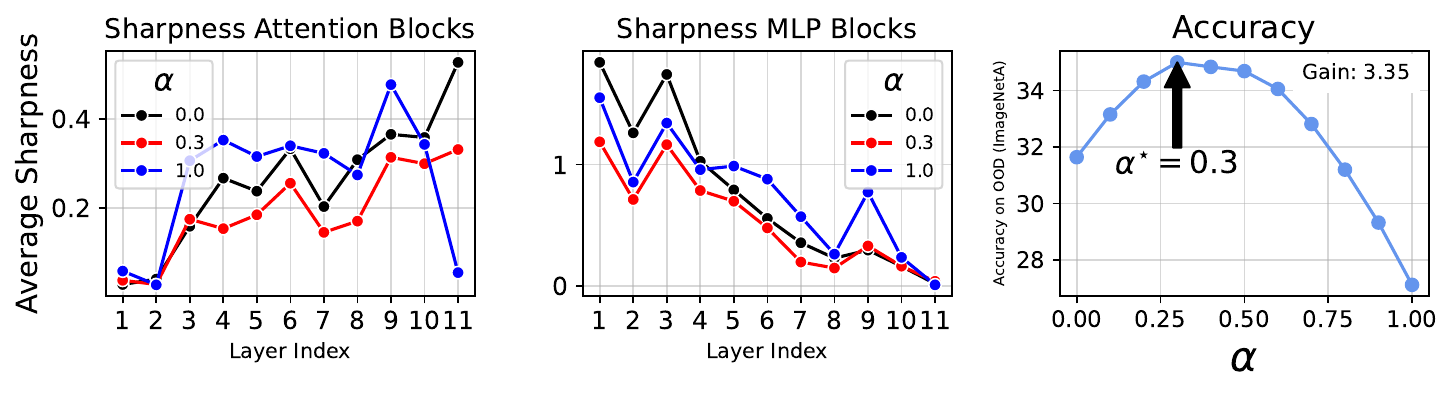} &
		\includegraphics[width=0.5\linewidth]{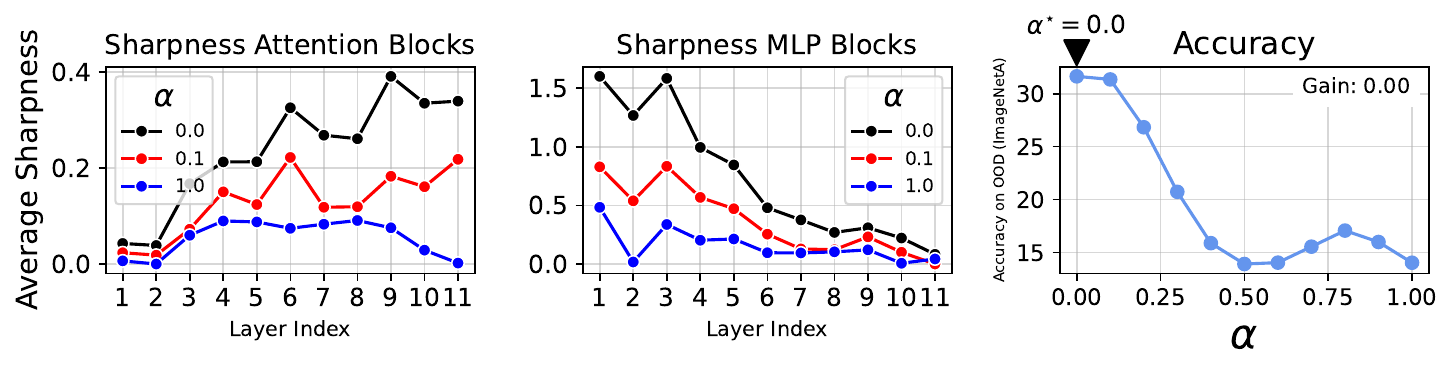} \\
	\end{tabular}
	\vspace{-3mm}
	\caption{We present an analysis of the \textit{layer-wise} sharpness across four distinct CLIP models, comprising two \textit{failure mode} models and two \textit{high gain accuracy} models, demonstrating the sharpness characteristics of each individual layer.} 
	\label{fig:layer_wise_sharpness_complete}
\end{figure*}
Figure \ref{fig:layer_wise_sharpness_complete} demonstrates that in models with \textit{high gain accuracy} (second row), the optimal $\alpha^{\star}$ corresponds to a point where the interpolated model achieves \textit{maximum} generalization performance. However, within this model, there is at \textit{least one layer} where the \textit{layer-wise} sharpness is \textit{nearly zero}. On the other hand, for \textit{failure mode} models, it is already known that there is no point along the interpolated path where the OOD accuracy surpasses that of the starting and ending points. Consequently, $\alpha^{\star}$ is exactly at the starting point (the \textit{zero-shot} point). For \textit{failure mode} models, it can be observed that there is at least one layer where the \textit{layer-wise} sharpness is nearly zero. In fact, fine-tuned \textit{failure mode} models inherently possess a \textit{straggler layer}. In the following section, we evaluate our \textit{layer-wise} sharpness and \textit{straggler} layer in a different direction. We introduce a straightforward algorithm based on the \textit{layer-wise} sharpness of the \textit{straggler} layers.

\myparagraph{On the role of Sparsity for Generalization and \texttt{RFT}.}
\vspace{-3mm}
\begin{algorithm}[H]
	\caption{Pytorch Pseudocode for Straggler Layer Pruning}
	\label{alg:sparse}
	\begin{algorithmic}[1]
		\REQUIRE Model $\mathcal{M}$ structured in $L$ layers $\{\vW^{(1)}, \dots, \vW^{(L)}\}$, \textit{zero-shot} model $\vth_{\texttt{zero-shot}}$.
		\FOR{$i = 1$ to $L$}
		\IF{\texttt{Adaptive Average Sharpness}($\vW^{(i)}$, $\rho$) $\simeq 0$}
		\STATE mask $\leftarrow$ \texttt{torch.bernoulli}(\texttt{torch.full\_like}( $\mathcal{M}$[$\vW^{(i)}$], 0.5)).\texttt{bool()}
		\STATE $\mathcal{M}$[$\vW^{(i)}$][mask] $\leftarrow 0$
		\ENDIF
		\ENDFOR
		\STATE $\vth_{\alpha}=$ \texttt{interpolation}($\vth_{\texttt{zero-shot}}$, $\mathcal{M}$)
		\RETURN $\vth_{\alpha}$
	\end{algorithmic}
\end{algorithm}
\vspace{-5mm}
Our objective is to establish a connection between the \textit{layer-wise} sharpness of \textit{straggler} layers and the generalization performance of the interpolated model. While the primary aim of this work is not to introduce a new algorithm that surpasses conventional interpolation methods, we focus on elucidating the importance of the \textit{layer-wise} sharpness phenomenon. First, through five iterations, we identify the \textit{straggler} layers of the fine-tuned CLIP model. Subsequently, we randomly adjust the weights of the identified layers. Specifically, before initiating the interpolation, we make the \textit{straggler} layers \textit{\textbf{sparse}}. In Algorithm~\ref{alg:sparse}, we summarize our algorithm.

\begin{figure}[t!] \centering\small
	\tabcolsep=1.1pt
	\newl=.495\columnwidth
	\begin{tabular}{c c} 
		\hspace{3mm} \textbf{Model 1} & \hspace{3mm} \textbf{Model 2} \\
		\includegraphics[width=\newl]{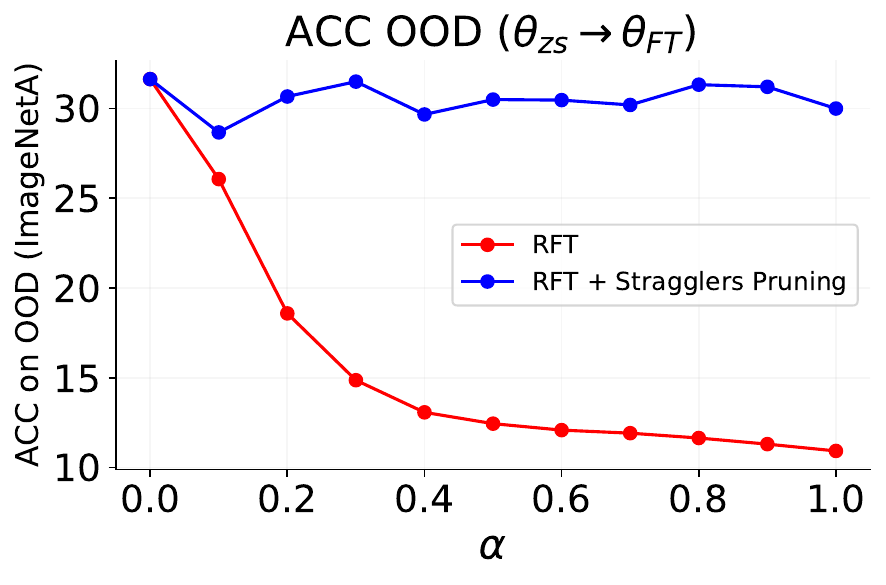} &
		\includegraphics[width=\newl]{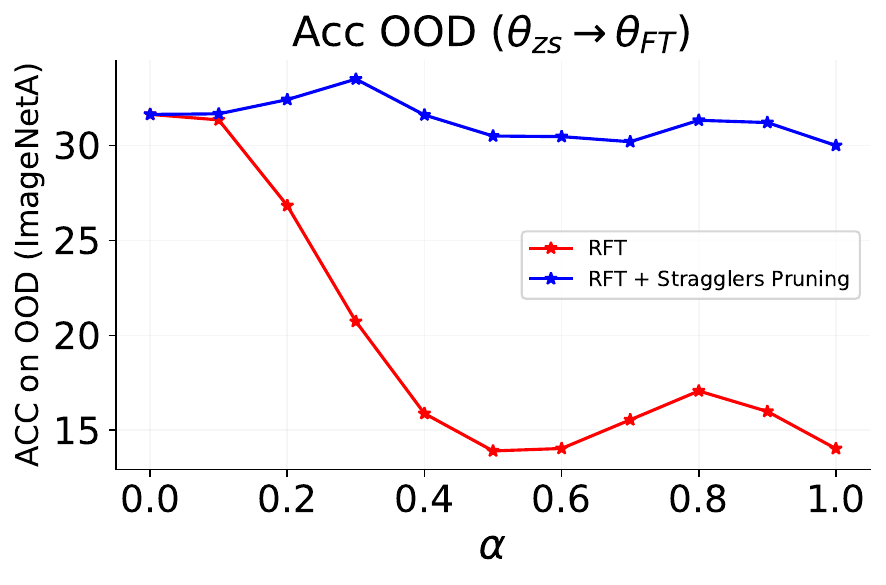}\\
	\end{tabular}
	\vspace{-5.5mm}
	\caption{\textbf{Straggler layer pruning.} For two distinct \textit{fine-tuned} CLIP models that exhibit \textit{failure mode} during interpolation using the \texttt{RFT} algorithm, we demonstrate that pruning the \textit{straggler} layers of the \textit{fine-tuned} model prevents a collapse in performance.} 
	\label{fig:pruning}
\end{figure}

\section{Conclusion and Future works}
In conclusion, our study underscores the critical role of interpolation~(\texttt{RFT}) in enhancing the generalization capabilities of CLIP models for OOD tasks. We demonstrate that by putting specific layers in CLIP models under the microscope, referred to as \textit{straggler layers}, and employing the concept of \textit{layer-wise sharpness} as opposed to the traditional notion of \textit{general sharpness}, we can effectively assess the generalization performance of these interpolated models on OOD data. Our findings indicate that if a \textit{fine-tuned} CLIP model contains at least one layer where the \textit{layer-wise sharpness} is nearly zero, it triggers a \textit{failure mode} phenomenon. Furthermore, for interpolated CLIP models that achieve \textit{high gain accuracy} along the interpolation path, a decline in OOD performance begins when, at the point of maximum OOD accuracy~($\alpha^{\star}$), there exists a layer with nearly zero \textit{layer-wise} sharpness. This specific layer is identified as the \textit{straggler layer}. Importantly, this study is the first to explore the generalization and interpretability of CLIP models, through the lenses of mode connectivity, interpolation and \textit{sharpness}. Our findings provide novel insights into the behavior of these models and their potential for robust application across diverse tasks.

\newpage
\bibliographystyle{plainnat}
\bibliography{example_paper.bib}

\newpage
\appendix

\section{Appendix}
\label{sec:app_asymptotic}
Following to \citep{andriushchenko2023modernlookrelationshipsharpness}, let $L_\Scal(\wv) = \frac{1}{|S|} \sum_{(\xv, \yv) \in \Scal} \l_{\xv\yv}(\wv)$ be the loss on a set of points $\Scal$.
For arbitrary weights $\wv$ (i.e., not necessarily a minimum), then the \textit{average-case sharpness} is defined as:
\begin{align} %
	S_{avg,p}^\rho(\wv, \cv) \triangleq \E_{\substack{\Scal \sim P_m \ \ \ \ \ \ \\ \deltav \sim \mathcal{N}(0, \rho^2 diag(\cv^2))}} \hspace{-8mm} L_\Scal(\wv + \deltav) - L_\Scal(\wv) \nonumber
\end{align}
where $\odot$/$^{-1}$ denotes elementwise multiplication/inversion and $P_m$ is the data distribution that returns $m$ pairs $(\xv, \yv)$.

If $\cv=|\wv|$ then the perturbation set is $\norm{\delta \odot |\wv|^{-1}}_p \leq \rho$. Assume a new variable $\gammav=\deltav \odot |\wv|^{-1}$ and perform a Taylor expansion around $w$:
\begin{align*}
	L_\Scal(\wv+\deltav) = L_\Scal(\wv+\gammav \odot |\wv|)
	= L_\Scal(\wv) + \inner{\nabla L_\Scal(\wv), |\wv| \odot \gammav} + \frac{1}{2}\inner{\gammav \odot |\wv|, \nabla^2 L_\Scal(\wv) \gammav \odot |\wv|} + O(\norm{\gammav}_p^3),
\end{align*}
where $\nabla^2 L_\Scal(\wv)$ denotes the Hessian of $L_\Scal$ at $\wv$.

\begin{proposition} (\cite{andriushchenko2023modernlookrelationshipsharpness}),
	Let $L_\Scal\in C^3(\R^s)$, $S$ be a finite sample of points $(x_i,y_i)_{i=1}^n$ and let $P_m$ denote the uniform distribution over subsamples of size $m\leq n$ from $S$. %
	Then
	\begin{align*}
		\lim_{\rho \to 0} \frac{2}{\rho^2} S_{avg}^\rho(\wv, |\wv|) 
		&=\E_{\Scal \sim P_m} \left[ \tr(\nabla^2 L_\Scal(\wv) \odot |\wv| |\wv|^\top) \right] + O(\rho)
	\end{align*}
\end{proposition}
\begin{proof} Let us consider the loss without the subcript for clarity. Then we consider
	\[  \E_{\deltav \sim \mathcal{N}(0, \rho^2 diag(\cv^2))} L_\Scal(\wv + \deltav) - L_\Scal(\wv) \]
	When plugging in the Taylor expansion of the loss, we see that 
	\begin{align*}  \E_{\deltav \sim \mathcal{N}(0, \rho^2 diag(\cv^2))} &L_\Scal(\wv + \deltav) - L_\Scal(\wv)\\
		=& \E_{\gammav \in \mathcal{N}(0, \rho^2 \vI)} \Big[ \inner{\nabla L_\Scal(\wv), |\wv| \odot \gammav} + \frac{1}{2}\inner{\gammav \odot |\wv|,\nabla^2 L_\Scal(\wv) \gammav \odot |\wv|} + O(\norm{\gammav}_2^3)\Big]\\
		=&\frac{1}{2}\E_{\gammav \in \mathcal{N}(0, \rho^2 \vI)} 
		\Big[ \inner{\gammav \odot |\wv|,\nabla^2 L_\Scal(\wv) \gammav \odot |\wv|} \Big] + O(\rho^3)\\
		=&\frac{1}{2}\E_{\gammav \in \mathcal{N}(0, \rho^2 \vI)} \Big[\inner{\gammav, \big(\nabla^2 L_\Scal(\wv) \odot |\wv| |\wv|^T\big) \gammav}\Big] + O(\rho^3)\\
		=&\frac{\rho^2}{2}\tr(\nabla^2 L_\Scal(\wv) \odot |\wv| |\wv|^\top) + O(\rho^3) 
	\end{align*}
\end{proof}


\end{document}

%% file: latex_defs.tex
\def\R{\mathbb{R}}

\newcommand{\norm}[1]{\left\|#1\right\|}

\newcommand{\inner}[1]{\left\langle#1\right\rangle}

\DeclareMathOperator{\E}{\mathbb{E}}
\let\l\relax
\DeclareMathOperator{\l}{\mathbf{\ell}}

\newcommand{\cv}{\boldsymbol{c}}

\newcommand{\wv}{\boldsymbol{w}}
\newcommand{\xv}{\boldsymbol{x}}
\newcommand{\yv}{\boldsymbol{y}}

\newcommand{\vI}{\boldsymbol{I}}

\newcommand{\vW}{\boldsymbol{W}}

\newcommand{\tr}{\text{tr}}

\newcommand{\deltav}{\boldsymbol{\delta}}
\newcommand{\gammav}{\boldsymbol{\gamma}}

\newcommand{\Scal}{\mathcal{S}}

\newlength\newl